\documentclass[10pt,journal,compsoc]{IEEEtran}
\ifCLASSOPTIONcompsoc
\usepackage[nocompress]{cite}
\else
\usepackage{cite}
\fi

\ifCLASSINFOpdf
\usepackage[pdftex]{graphicx}
\graphicspath{{./figs/}}
\DeclareGraphicsExtensions{.pdf,.jpeg,.png,.eps}
\else

\fi
\usepackage{amsmath}
\usepackage{latexsym}
\usepackage{amssymb}
\usepackage{algorithmic}
\usepackage{algorithm}
\usepackage{array}
\usepackage[table,xcdraw]{xcolor}
\usepackage{balance}
\usepackage{url}
\usepackage{parskip}
\usepackage{hyperref}
\usepackage{color}
\usepackage{bm}
\usepackage{setspace}
\usepackage{fancyhdr}
\newcommand{\subparagraph}{}
\usepackage[stable]{footmisc}
\usepackage{csquotes}
\MakeOuterQuote{"}

\ifCLASSOPTIONcompsoc
\usepackage[caption=false,font=footnotesize,labelfont=sf,textfont=sf]{subfig}
\else
\usepackage[caption=false,font=footnotesize]{subfig}
\fi

\makeatletter
\newcommand{\vast}{\bBigg@{2.6}}
\newcommand{\Vast}{\bBigg@{4.6}}
\newcommand{\ignore}[1]{}
\makeatother

\newcommand{\algorithmfootnote}[2][\footnotesize]{%
	\let\old@algocf@finish\@algocf@finish% Store algorithm finish macro
	\def\@algocf@finish{\old@algocf@finish% Update finish macro to insert "footnote"
		\leavevmode\rlap{\begin{minipage}{\linewidth}
				#1#2
		\end{minipage}}%
	}%
}
\newtheorem{theorem}{Theorem}
\newtheorem{lemma}{Lemma}

\providecommand{\red}{\textcolor[rgb]{1.00,0.00,0.00}}

\begin{document}
	
	\title{Task Allocation for Asynchronous Mobile Edge Learning with Delay and Energy Constraints}
	
		\author{Umair~Mohammad,~\IEEEmembership{Student Member,~IEEE,}
		Sameh~Sorour,~\IEEEmembership{Senior Member,~IEEE,}
		and Mohamed~Hefeida,~\IEEEmembership{Senior Member,~IEEE}% <-this % stops a space
		\IEEEcompsocitemizethanks{
			\IEEEcompsocthanksitem U. Mohammad is with the School of Computing and Information Science, Florida International University, Miami, FL, USA, 33172. \protect \\ E-mail: umair.mohammad@fiu.edu%\protect\\
			\IEEEcompsocthanksitem S. Sorour is with the School of Computing, Queen’s University, Kingston,ON, Canada, K7L 3N6. \protect \\ E-mail: sameh.sorour@queensu.ca
			\IEEEcompsocthanksitem M. Hefeida is with the Department of Computer Science and Electrical Engineering, West Virginia University, Morgantown, WV, USA, 26506. \protect \\ E-mail: mohamed.hefeida@mail.wvu.edu%\protect\\
			}
	}
	
	\markboth{Journal of \LaTeX\ Class Files,~Vol.~14, No.~8, August~2015}%
	{Shell \MakeLowercase{\textit{et al.}}: Bare Advanced Demo of IEEEtran.cls for IEEE Computer Society Journals}

	\IEEEtitleabstractindextext{%
		\begin{abstract}
This paper extends the paradigm of "mobile edge learning (MEL)" by designing an optimal task allocation scheme for training a machine learning model in an asynchronous manner across mutiple edge nodes or learners connected via a resource-constrained wireless edge network. The optimization is done such that the portion of the task allotted to each learner is completed within a given global delay constraint and a local maximum energy consumption limit. The time and energy consumed are related directly to the heterogeneous communication and computational capabilities of the learners; i.e. the proposed model is heterogeneity aware (HA). Because the resulting optimization is an NP-hard quadratically-constrained integer linear program (QCILP), a two-step suggest-and-improve (SAI) solution is proposed based on using the solution of the relaxed synchronous problem to obtain the solution to the asynchronous problem. The proposed HA asynchronous (HA-Asyn) approach is compared against the HA synchronous (HA-Sync) scheme and the heterogeneity unaware (HU) equal batch allocation scheme. Results from a system of 20 learners tested for various completion time and energy consumption constraints show that the proposed HA-Asyn method works better than the HU synchronous/asynchronous (HU-Sync/Asyn) approach and can provide gains of up-to 25\% compared to the HA-Sync scheme. 
		\end{abstract}
		
		}
	
	\maketitle
	\IEEEpeerreviewmaketitle	
	\IEEEdisplaynontitleabstractindextext
	
	\ifCLASSOPTIONcompsoc
	\IEEEraisesectionheading{\section{Introduction}\label{Section1__Introduction}}
	\else
	\section{Introduction}
	\label{Section1_Introduction}
	\fi
	Mobile edge computing (MEC) is proving to be a successful paradigm for dealing with the computational challenges raised by the era of the internet of everything (IoE). With the world embracing smart architecture including smart cities, smart grids, smart homes, etc., 41 billion internet of things (IoT) devices are expected be connected to the internet by 2022 \cite{N202003_vXchnge_IoT_data}. Examples of end user devices/edge nodes include smart phones, traffic cameras, autonomous connected vehicles, unmanned aerial vehicles (UAVs), etc. Furthermore, Cisco estimates that this equipment will generate up-to 800 zettabytes ($\sim 10^{21}$) of data \cite{N202003_vXchnge_IoT_data}. 
	
	To serve a useful purpose, this data needs to be analyzed. 
	Transferring these monumental amounts of data generated on devices mostly connected via wireless edge networks spread across vast geographical regions to cloud servers for analysis via multiple backhaul links is time-consuming, costly, and raises security and privacy concerns \cite{Ref1_Original}. Therefore, it is expected that 90\% of data analytics will be done on either edge processors using MEC or on the end devices themselves using hierarchical MEC (H-MEC) \cite{Ref2_Original}. 
	
	This paradigm of task computation at the edge is supported by the latest works in MEC/H-MEC \cite{P0_Letaief_Survey, RefsPap_MEC126, Liu2017a,Moha1812:Multi}. It is expected that machine learning (ML) tasks will comprise a large part of H-MEC computations because ML has shown to provide superior performance in many data analytics applications such as predictive modeling, object recognition and image segmentation. Such applications are expected to form the basis of Edge Artificial Intelligence (Edge AI).
	
	Many ML techniques, including regression, support vector machine (SVM) and neural networks (NN) are built on gradient-based learning.
	This usually involves iterative updates of the ML model parameters based on the gradient of a loss function that is defined according to the ML model. Because such iterative approaches can be exhaustive for a single device, distributed learning (DL) has been proposed to training ML algorithms over multiple learners. There are two approaches possible: training an ML model on a large dataset in a distributed manner where subsets of the data are located across multiple learners (data parallelism) or to train very large models distributedly on a single dataset located at each learner (model parallelism). Although both options apply to the wireless edge and our work does support MP, most of the discussion focuses on the DP scenario. 
	However, MP support will be highlighted whenever appropriate. 
		
	DL has been widely investigated over wired/non-heterogeneous computing and communication environments \cite{GoogleDistWired, StalenessAwarePaper,J1704_WiredCommOptDistSVM_TPDS_YouVuduc,J1811_WiredEthSparkDistDL_TPDS_LangerRahayu}. Recently, researchers have turned their attention to deploying DL models for training on nodes or learners connected via the wireless edge \cite{Wang2019,Tuor_01_AdaptiveControl,Tuor_02_2018_DemoAbstract,Tuor04,Mohammad2019a,Mohammad2019_AsynArxiv,Moha2020b,2019arXiv190907972C,2019arXiv191102417Y,AsyncFedOpt_Proof,J1909_TPDS_DynSelTun_KarnElfadel,J201905_CommEffAlgSVM_TPDS_DassRabi}. 	
	DL over the wireless edge, or mobile edge learning (MEL) as we term it, is motivated by two distinct yet practical scenarios: federated learning (FL) and parallelized learning (PL). 
	
	In FL, the learners own their local dataset. A central hub called an orchestrator initiates the learning process by sending the global model parameter set to each learner, where each learner, in parallel, performs the updates on the ML model locally using its private dataset, and sends back the locally updated model. The orchestrator collects the local ML models from each learner, does the global update/aggregation, and sends back the optimal ML model to each learner for the next cycle until a stopping criteria is reached. We refer to one such cycle as the global update cycle. Typically, FL has been proposed in literature to preserve data privacy \cite{HarvardDistributed,J1808_ModelPrivacy_TPDS_JiaFang}. 
	
	On the other hand, PL usually involves an orchestrator that parallelizes the learning process on its locally owned dataset over multiple learners. The difference in PL is that at the start of each global cycle, the orchestrator also sends a subset of the data on which to learn along with the model parameters. This scenario is mainly useful for cases where the orchestrator may have enough data but not the processing capability to learn on such a large dataset. It can leverage the communication and computation capacity of other learners to train the ML model distributedly. The two distinct approaches are summarized in Figure \ref{figure01blabel}.	
	\begin{figure}[t!]
		\centering
		\includegraphics[width=\linewidth,keepaspectratio]{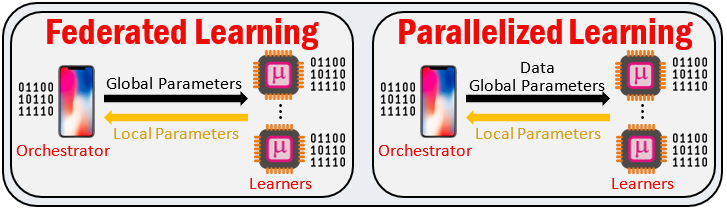}
		\caption{Illustration of the two distinct approaches of FL and PL.}
		\label{figure01blabel}
	\end{figure}
	
	Deploying ML models for MEL presents unique challenges because the wireless edge is resource-constrained from various aspects. For example, communication resources are limited in terms of bandwidth and available channels. Most end devices are battery operated and therefore limited in energy. Simultaneously, the tasks may need to be completed within very short times. In contrast to multicore processor platforms connected via wired networks such as high performance computing (HPC) systems, the different nodes in edge networks have widely varying computing and communication capabilities. Therefore, it is vital to study the impact of these constraints and heterogeneous resources on MEL. Some of the factors affecting the loss or accuracy of the ML model in MEL include the number of total updates (including cycles at which global aggregation occurs), frequency of global updates, and the data distribution at each learner. These parameters are constrained by the available resources.   
	
	Although FL has been extensively studied in literature \cite{Wang2019,Tuor_01_AdaptiveControl,Tuor_02_2018_DemoAbstract,Tuor04,2019arXiv190907972C,2019arXiv191102417Y}, PL in particular or in general, MEL which comprises both, have been sparsely studied \cite{Mohammad2019a,Moha2020b,Mohammad2019_AsynArxiv}. Some works have proposed algorithms for edge DL without specifically focusing on the resource allocation issues \cite{J1909_TPDS_DynSelTun_KarnElfadel,J201905_CommEffAlgSVM_TPDS_DassRabi}. 		
	Most works on edge FL that consider resource consumption only do so in generic terms \cite{Wang2019,Tuor_01_AdaptiveControl,Tuor_02_2018_DemoAbstract,Tuor04} without considering the heterogeneities, i.e. they are heterogeneity unaware (HU). Although the works of \cite{2019arXiv190907972C,2019arXiv191102417Y} are heterogeneity aware (HA), they ignore the aspect of batch allocation and the PL scenario. The implications of wireless computation/communication heterogeneity on optimizing batch allocation to different learners for maximizing accuracy while satisfying a delay constraint were studied in \cite{Mohammad2019a} and dual time and energy constraints in \cite{Moha2020b}.
	
	Typically in MEL, the orchestrator waits for all learners to complete an equal number of iterations of the ML training algorithm and hence, we call this the synchronous approach. The above referenced works mainly deal with the synchronous approaches such as the adaptive control algorithm of \cite{Wang2019} and the HA-Sync scheme of \cite{Mohammad2019a}. 
	Recently, some work has been carried out on the asynchronous scenario by allowing some staleness between the local model parameters and the locally calculated gradients so that powerful devices with good communication links may actually provide a faster validation accuracy progression. This has been proposed for wired networks \cite{StalenessAwarePaper} and for training over the wireless edge \cite{AsyncFedOpt_Proof, Mohammad2019_AsynArxiv}. However, the work in \cite{AsyncFedOpt_Proof} does not cover the impact of the physical layer whereas \cite{Mohammad2019_AsynArxiv} does not include the impact of energy constraints. 
	
	To the best of the authors' knowledge, this work is the first attempt to have a staleness aware algorithm for asynchronous MEL with both, time and energy constraints. Here, we emphasize that our model is different than the models in \cite{StalenessAwarePaper, AsyncFedOpt_Proof} such that the system in our proposed asynchronous scheme, learner performs still reports back the local model parameters within a pre-set duration. However, individually ,the learners may have performed a different number of local updates within this duration and therefore, the system is asynchronous in terms of the number of updates. This approach will make sure that the aggregation is done uniformly for all learners without being affected by stragglers or bad-performing learners. The novelty is in the fact that the number of local updates and the local dataset size on which a learner trains the local model are jointly optimized. Furthermore, it is guaranteed that the task will complete within a given duration without exceeding local energy consumption limits for each learner. 
	
	To this end, the problem for MEL is first formulated as quadratically-constrained integer linear problem (QCILP), which is an NP-hard problem. 
	It is shown that FL is just a subset of the PL approach with the component of batch re-transmission from the orchestrator to each learner removed. 
	Thus, the MEL model discussion focuses on the more general PL scenario but all variations for the FL scenario will be clarified whenever needed. 
	A two-step solution is then proposed based on a relaxation and suggest and improve (SAI) approach. 
	The merits of the proposed solution are shown by comparing its performances to both,the HA-Sync and the HU equal task allocation approach of \cite{Tuor_01_AdaptiveControl,Tuor_02_2018_DemoAbstract}. 
	
	\subsection{Contributions}
	
	This paper extends the work on MEL in the following ways:
	\begin{enumerate}
		\item In contrast to the work on asynchronous MEL in \cite{Mohammad2019_AsynArxiv} and the work on dual time and energy constraints for synchronous MEL \cite{Moha2020b}, this paper provides solutions to facilitate asynchronous MEL with dual time and energy constraints while being HA.
		\item A new way to model the problem for asynchronous MEL is proposed in Section \ref{Sec04__Problem}. In contrast to \cite{Mohammad2019_AsynArxiv}, the objective is now to maximize local updates while controlling staleness. A comprehensive proof on the benefits of maximizing the average number of local updates while controlling the staleness is provided.
		\item A novel two-step solution is proposed in Section \ref{Section5__Solution} where the loose upper bounds on the solution to the synchronous problem are used as initial conditions to the asynchronous problem (HA-Asyn).
		\item Through extensive simulations under varying time and energy constraints, the impact of different levels of staleness is studied and the gains achieved compared to the HA-Sync scheme of \cite{Moha2020b} are quantified in Section \ref{Sec06__SimulationResults}. Moreover, it is shown that in some cases, the proposed HA-Asyn provides a better solution to the HA-Sync problem, and that HA schemes in general work better than the HU scheme.
		\item Section \ref{Sec06__SimulationResults} also summarizes the results by providing recommendations on the best scheme for different scenarios and more importantly, suggestions are given on selecting the best staleness level for the HA-Asyn scheme.
	\end{enumerate}
	
	The rest of the paper is organized as follows: Section \ref{Section03__SystemModelParameters} presents system model including the general DL followed by a transition to MEL. The problem of interest in this paper is formulated in Section \ref{Sec04__Problem} and our proposed solution for this problem is detailed in Section \ref{Section5__Solution}. Section \ref{Sec06__SimulationResults} illustrates the testing results and Section \ref{Sec07__Conclusion} concludes the paper. 
	
	\section{System Model for MEL}
	\label{Section03__SystemModelParameters}
	\subsection{Gradient-based DL Preliminaries}
	Consider a dataset $\mathcal{D}$ that consists of $d$ samples which can be used to train an ML model where each sample $n$ for $n=1,\ldots,d$ has a set of $\mathcal{F}$ features denoted by $\mathbf{x}_n$ and a target $y_n$. In ML, the objective is to find the relationship between $\mathbf{x}_n$ and $y_n$ using a set of parameters $\mathbf{w}$ such that a loss function, $F\left(\mathbf{x}_n, \mathbf{y}_n, \mathbf{w}\right)$, or $F_n\left(\mathbf{w}\right)$ for short because $\mathbf{x}_n$ and $y_n$ are known, is minimized. Because it is generally difficult to find an analytical solution, typically an iterative gradient descent approach is used to optimize the set of model parameters such that $\mathbf{w}[l+1] = \mathbf{w}[l] - \eta\nabla F\left(\mathbf{w}[l]\right)$ where $l$ represents the time step or iteration and $\eta$ is the learning rate typically set on the interval $(0,1)$. 	
	In deterministic gradient descent (DGD), the ML model goes over each sample one-by-one, or more commonly, batch-by-batch using a mini-batch approach, until it reaches sample \# $d$; completing one epoch. 
	If data is re-shuffled randomly in between epochs, this method is known as stochastic GD (SGD). A total of $L$ epochs may be performed depending on the stopping criteria. 
	
	\begin{figure}[t!]
		\centering
		\includegraphics[width=1\linewidth]{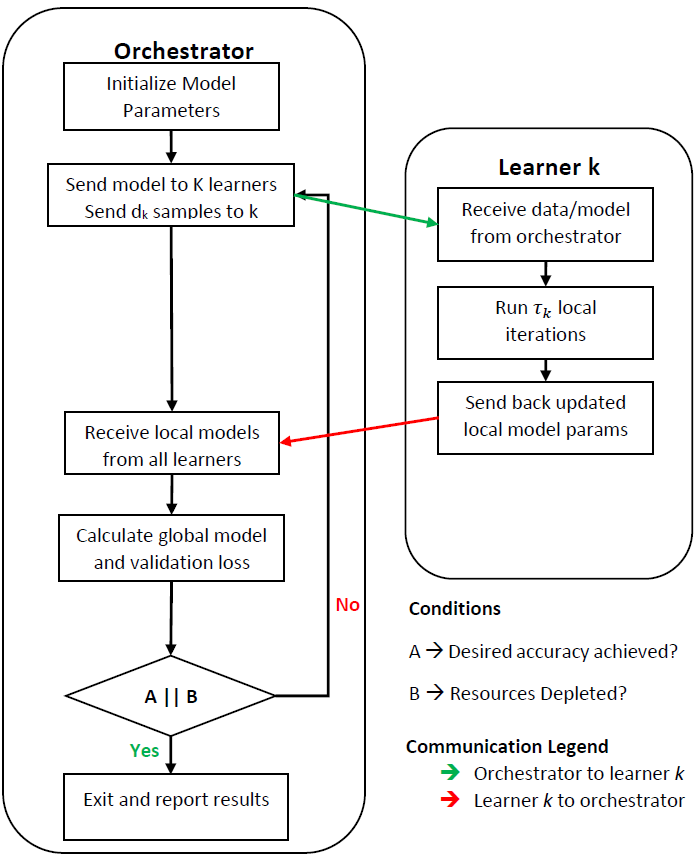}
		\caption{Illustration of the DL process with DP}
		\label{figure01label}
	\end{figure}

	Consider the case where there exists one centralized controller or orchestrator that trains an ML model to solve a specific problem (classification, prediction, image segmentation, etc.) on a set of $\mathcal{K} = \{1,~\ldots,~k,~\ldots,~K\}$ learners. In DL with DP, a batch of the data $\mathcal{D}_k$ of size $d_k$ is present at each individual learner $k$ which may be locally owned or supplied by the orchestrator. The orchestrator initiates the learning process by sending a global model $w$ and possibly $d_k$ samples to each learner $k$. Each learner $k$ applies the gradient descent approach to the local model $\mathbf{w}_k$ in parallel as shown in (\ref{Eq_Sec03_LocalModel}), and sends back the local models to the orchestrator for global aggregation. One such cycle is can be called the global update cycle.
	\begin{equation}
	\mathbf{w}_k[l+1] = \mathbf{w}_k[l]-\eta \nabla F_k(\mathbf{w}_k[l]) 
	\label{Eq_Sec03_LocalModel}
	\end{equation}
	The local model parameter at learner $k$ is given by $\mathbf{w}_k$, the local loss is given by $F_k$, and $\eta$ is the learning rate. At time-step $l+1$, the local model $\mathbf{w}_k[l+1]$ depends on the model $\mathbf{w}_k[l]$ at previous step $l$ and the gradient of the local loss $\nabla F_k(\mathbf{w})$. The local loss $F_k ~\forall ~ k\in\mathcal{K}$ can be calculated using the local dataset $\mathcal{D}_k$ of size $d_k$ in the following way \cite{Wang2019}:
	\begin{equation}
	F_k(\mathbf{w}) = \dfrac{1}{d_k} \sum_{n=1}^{d_k} F_n(\mathbf{w})
	\label{Eq_Sec03_LocalLoss}
	\end{equation} 
	
	However, the global optimal model parameter set $\mathbf{w}$ will only be visible to the learners after a global aggregation which may occur at any arbitrary time-step $l$ for $l = 1,~\ldots,~L$. In the synchronous case, at all learners $k$, a global aggregation occurs after $\tau$ time-steps; at that instant, $\mathbf{w}_k = \mathbf{w} ~ \forall ~ k \in \mathcal{K}$. In the asynchronous case, the global aggregations will occur after potentially different $\tau_k$ updates for each learner $k ~ \forall ~ k\in \mathcal{K}$. The globally optimal model parameter set can be obtained by applying the following aggregation mechanism \cite{Wang2019}:
	\begin{equation}
	\mathbf{w} = \dfrac{1}{d}\sum_{k=1}^{K}{d_k \mathbf{w}_k}  
	\label{Eq_Sec03_GlobalAgg}
	\end{equation}
	The orchestrator may perform multiple global cycles until a stopping criteria is reached such as an accuracy threshold or resource depletion. This process is summarized in Fig. \ref{figure01label}.

	\subsection{Transition to MEL}
	An MEL system consists of an orchestrator and $K$ learners where $d_k$ data samples allocated to learner $k$, $k \in \kappa = \{ 1, 2, \dots, K\}$ so that it performs $\tau_k$ learning iterations. 
	Each learner has a computational capacity of $f_k$ in Hz, an associated communication channel $h_{k0}$ to the orchestrator. An example of an MEL system is illustrated in fig. \ref{figure02label}. 	
	\begin{figure}[t!]
		\centering
		\includegraphics[width=1\linewidth]{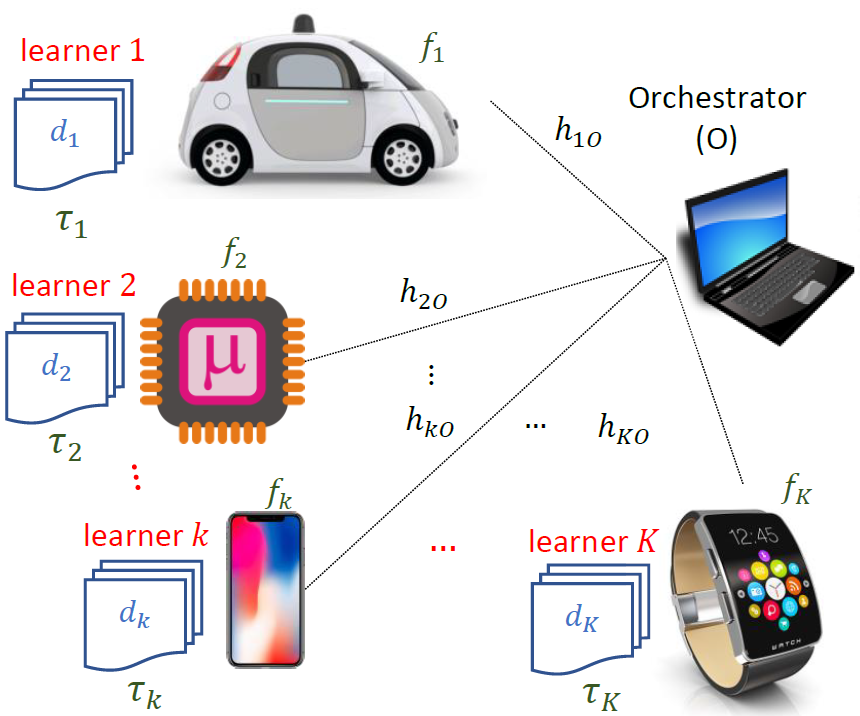}
		\caption{System model of a MEL setting}
		\label{figure02label}
	\end{figure}

	In MEL, one approach is to have all $K$ learners do $\tau_k = \tau$ local updates of the model parameters $\mathbf{w}_k ~ \forall k \in \mathcal{K}$; this is the synchronous approach presented in \cite{Moha2020b}. The other approach is to optimize $\tau_k$ for each learner while controlling the difference among the number of local updates done by different learners. Here, we define staleness $s_{k,l}$ as the difference between the number of updates done by two arbitrary learners $k$ and $l$ in the following way:
	   \begin{equation}
	   s_{k,l} = \tau_k - \tau_l \quad \mbox{if} \begin{cases} k ~ \& ~ l ~ \in \mathcal{K} \mid l\neq k  \\ \tau_k \geq \tau_l \end{cases}
	   \label{eq_staleness_definition}
	   \end{equation}
		
	Furthermore, as described in section \ref{Section1__Introduction}, there are two possibilities for MEL. In FL, the orchestrator and each learner $k$ exchange model parameters only whereas in PL, the orchestrator also sends $d_k$ data samples along with the global model $\mathbf{w}$ at the beginning of each global cycle. It can be noted here that FL is a subset of PL where the orchestrator only sends model parameters instead of both, data and model parameters. The remaining steps are the same as described in Fig. \ref{figure01label}.
	
	Given the above descriptions of the learning process, the orchestrator requires $t_k^S$ seconds to transmit the model parameters and/or the batch $d_k$ for learner $k$. Please note that the energy consumed by the orchestrator during transmission is ignored because we are concerned with the energy consumption of the learners only. Learner $k$ consumes $t_k^C$ seconds and energy $e_k^C$ to perform one learning iteration over it's allotted dataset $d_k$. It then consumes a further time $t_k^R$ and energy $e_k^R$ to send back the updated local model parameters $\mathbf{w_k}$. Therefore, for each learner, the time taken and the energy consumed for the local learning process is given by (\ref{Eq_1_timeForLearnerk}) and (\ref{Eq_2_energyForLearnerk}), respectively.
	\begin{equation}
	\label{Eq_1_timeForLearnerk}
	t_k = t_{k}^S+\tau_k t_k^C + t_k^R   
	\end{equation}
	\begin{equation}
	\label{Eq_2_energyForLearnerk}
	e_k = \tau_k e_k^C + e_k^R   
	\end{equation}
	The size of the model parameter set $\mathbf{w}_k$ is given by $B_k^{model} = \mathcal{P}_m\left(d_kS_d+S_m\right)$, where $\mathcal{P}_m$ is the model bit precision. $\mathcal{P}_m$ and $\mathcal{P}_d$ may include compression ratio for efficient storage.  $S_m$ represents the constant model size in terms of the ML model parameters whereas the term $d_kS_d$ is to support MP where $S_d$ is used to relate the model to the batch-size. 
	Each learner $k$ sends its locally updated model of size $B_k^{model}$ to the orchestrator with power $P_{ko}$ over a wireless channel of bandwidth $W$ and gain $h_{ko}$. Assuming that an ML model of computational complexity $C_m$, $X_k = d_k C_m$ computations occur per local update. 
	The time taken for learner $k$ to transmit the latest model parameters is given by: 
	\begin{equation}
	e_k^{R} = \dfrac{P_{k0}B_k^{model}}{R_{k}} =  \dfrac{\mathcal{P}_m \left(d_k\mathcal{S}_d+\mathcal{S}_m\right)}{W\log_2\left(1+\frac{P_{ko} h_{ko}}{N_0}\right)}, ~ k \in \mathcal{K}
	\label{eq:edge_energy_receiving}
	\end{equation}
	Let us represent learner $k$'s processing power by $f_k$ in GHz. The energy consumed by learner $k$ in each iteration of the learning iteration on a sample size of $d_k$ is given by:
	\begin{equation}
	e_k^{C} = \mu X_k {f_k}^{\nu- 1} = \mu d_k C_m {f_k}^{\zeta - 1}, ~ k \in \mathcal{K} \text{\cite{Moha1812:Multi}}
	\end{equation}
	where $\mu$ is the on-board chip capacitance (typically $10^{-11}$ F) and $\nu = 2$.
	
	The total energy consumed by learner $k$ in one global update cycle can be given by:
	\begin{align}
	\label{Eq_12_energyForLearnerk}
	e_k &= \tau_k e_k^C + e_k^R   \nonumber\\
	%= \dfrac{B_k^{data}+2B_k^{model}}{W\log_2\left(1+\frac{P_{ko} h_{ko}}{N_0}\right)}+ \tau \dfrac{X_k}{f_k}
	& = \dfrac{P_{k0}\mathcal{P}_m \left(d_k\mathcal{S}_d+\mathcal{S}_m\right)}{W\log_2\left(1+\frac{P_{ko} h_{ko}}{N_0}\right)} + \tau d_k \mu C_m {f_k}^{\zeta - 1}
	\end{align}	
	The total energy used by learner $k$ can be expressed as:
	\begin{equation}
	e_k = G_k^2  \tau_k d_k + G_k^1 d_k + G_k^0
	\label{Eq_16_G1isCompact}
	\end{equation}
	where the coefficients $G_k^2 = \mu C_m {f_k}^{\zeta - 1}$, $G_k^1 = \frac{P_{k0}\mathcal{P}_m\mathcal{S}_d}{W\log_2\left(1+\frac{P_{ko} h_{ko}}{N_0}\right)}$, and $G_k^0 = \frac{P_{k0}\mathcal{P}_m\mathcal{S}_m}{W\log_2\left(1+\frac{P_{ko} h_{ko}}{N_0}\right)} $.
		
	The size of local dataset $d_k$ can be given by $B_k^{data}= d_k\mathcal{F}\mathcal{P}_d ~ \forall ~ k$ where $\mathcal{F}$ is the feature vector size and $\mathcal{P}_d$ is the bit precision. At every global update step, each learner $k$ and the orchestrator will exchange $B_k^{model}$ bits and the orchestrator may send an additional and $B_k^{data}$ bits for PL. The total time taken $t_k$ can then be written as shown in (\ref{Eq_16_C1isCompact}). The coefficients are given by $C_k^2$, $C_k^1$ and $C_k^0$, respectively, where, $C_k^2 = \frac{\mathcal{C}_m}{f_k}$, $C_k^1 = \frac{\mathcal{F}\mathcal{P}_d+2\mathcal{P}_m\mathcal{S}_d}{W\log_2\left(1+\frac{P_{ko} h_{ko}}{N_0}\right)}$, and $C_k^0 = \frac{2\mathcal{P}_m\mathcal{S}_m}{W\log_2\left(1+\frac{P_{ko} h_{ko}}{N_0}\right)}$. The first term in $C_k^1$ is specific to PL. For more details on the formulation of $t_k$, the readers are referred to \cite{Mohammad2019a}.
	\begin{equation}
	t_k = C_k^2  \tau d_k + C_k^1 d_k + C_k^0
	\label{Eq_16_C1isCompact}
	\end{equation}

	\section{Problem Formulation}
	\label{Sec04__Problem}
	It has been previously shown that maximizing the local iterations per global cycle can lead to a faster progression of the learning process \cite{Mohammad2019a, Moha2020b}. On the other hand, for asynchronous approaches, it has been shown that accuracy can be optimized by controlling the staleness \cite{AsyncFedOpt_Proof, StalenessAwarePaper}. Although our model is different, \cite{StalenessAwarePaper,AsyncFedOpt_Proof}, we can still demonstrate that a joint approach be employed to obtain the best accuracy.
	\begin{lemma}
		Jointly controlling the staleness $s_{k,l} ~ \forall ~ k \in \mathcal{K} ~ \& ~ l ~ \in \{\mathcal{K} \mid l\neq k\}$ while maximizing the minimum number of updates $\min(\tau_k) k \in \mathcal{K}$ will minimize the global loss of the proposed HA-Asyn MEL. 
		\label{lemma1} 
	\end{lemma}

	\begin{proof}
		 Let's assume the orchestrator will train the MEL model for a total of $L$ epochs where a global aggregation can occur at any update step $l$ for $l = 1,~\ldots,~L$. Local updates occur at each step $l$. In the synchronous model of \cite{Wang2019,Mohammad2019a,Moha2020b}, between any two global updates, each learner $k$ performs $\tau$ updates whereas it performs $\tau_k$ updates in the proposed asynchronous version. Let us assume, to facilitate the analysis, that the global aggregations occur at integer multiples of the $\tau_m = \max(\tau_k)$; which represents the maximum possible updates that would be done by the highest performing learner. We can now define the interval $[g]$ defined over $[g\tau_m,~(g+1)\tau_m]$ for $g = 0,1,2,\ldots$. 
		
		Assuming a global aggregation were to occur at each iteration $l$, let us define an auxiliary set of global model parameters denoted by $\hat{\mathbf{w}}_{[g]}$ for any interval $[g]$ which would be calculated if a global update step took place. (Note that at the beginning of any interval $[g]$, a global aggregation does occur.) Then, the update rule for this auxiliary set can be given by:
		\begin{equation}
		\hat{\mathbf{w}}_{[g]}[l+1] = \hat{\mathbf{w}}_{[g]}[l]-\eta \nabla F(\hat{\mathbf{w}}_{[g]}[l]) 
		\end{equation} 
		
		We assume that the local loss function at learner $k$ given by $F_k(\mathbf{w})$ is:
		\begin{enumerate}
			\item convex
			\item $\rho$-Lipschitz $\lVert F_k(\mathbf{w})-F_k(\bar{\mathbf{w}}) \rvert \leq \rho \lvert \mathbf{w}-\bar{\mathbf{w}} \rvert$
			\item $\beta$-smooth $ \lVert \nabla F_k(\mathbf{w})-\nabla F_k(\bar{\mathbf{w}}) \rvert \leq \beta \lvert \mathbf{w}-\bar{\mathbf{w}} \rvert$
		\end{enumerate}
		These assumptions will hold for ML models with convex loss function such as linear regression and SVM. By simulations, we will show that the proposed solutions work for non-convex models such as DNN with ReLU activation. It has been shown that for such a model, the difference between the global optimal model and the auxiliary model for any iteration $l$ within an interval $g$, for $l = 1,\ldots,L$ and $g = 0,1,2,\ldots$, can bounded by:
		\begin{multline}
		\left\lVert \mathbf{w}[l+1]-\hat{\mathbf{w}}_g[l+1] \right\rVert \leq \left\lVert \mathbf{w}[l]-\hat{\mathbf{w}}_g[l] \right\rVert + \\ \dfrac{\eta\beta}{d}\sum_{k=1}^{K} f_k\left[l-g\tau_m\right]
		\label{staleness_div}
		\end{multline}
		
		The function $f_k(t) = \frac{\delta_k}{\beta}[(\eta\beta+1)^t-1]$ which relates the local model $\mathbf{w_k} ~ \forall ~ k \in \mathcal{K}$ to the auxiliary model set $\hat{\mathbf{w}}_g[l]$ as follows:
		\begin{equation}
		\left\lVert \mathbf{w}_k[l]-\hat{\mathbf{w}}_g[l] \right\rVert \leq f_k(l-g\tau_m)
		\label{staleness_div_local}
		\end{equation}
		
		Assume that each learner performs $\tau_k$ updates and for a particular interval, $l \in[g\tau_k,~(g+1)\tau_k] ~\forall ~ k\in \mathcal{K}$ where $l$ is the progression of the index of the best performing learner. Then, the upper bound on the model divergence can be given by the following expression:
		\begin{multline}
		\left\lVert \mathbf{w}[l+1]-\hat{\mathbf{w}}[l+1] \right\rVert \leq \left\lVert \mathbf{w}[l]-\hat{\mathbf{w}}[l] \right\rVert + \\ \dfrac{\eta\beta}{d}\sum_{k=1}^{K} f_k(l-g\tau_k)
		\label{staleness_div2}
		\end{multline}
		The learning rate $\eta$ can be selected such that $0\leq\eta\beta\leq1$ which is necessary to satisfy the assumptions in \cite{Wang2019}. In that case $1\leq\eta\beta+1\leq2$ and because $t_k = l-[g-1]\tau_k \geq0$, the function $f_k(t_k)$ grows exponentially greater as $t$ increases because the dominating term is $(\eta\beta+1)^{t_k}$. So, as the staleness $s_{k,l} = \tau_k - \tau_l \quad  k, l ~ \in \mathcal{K} \mid l\neq k$ increases, $f_k(t_k)$ will be higher for more learners which will result in a higher bound on the divergence. Therefore, as the auxiliary model diverges further from the globally optimal model, the loss will increase and hence, it can be expected that the accuracy will decrease.
		
		The learning will progress faster as $l$ is higher which can be done maximizing the $\tau_m$. Alternatively, if we want to keep $t_k ~ \forall ~ k$ low, we can maximize the $\min(\tau_k)$ while controlling staleness $s_{k,l}$. 
	\end{proof} 
	
	Therefore, the objective is to allocate batches $d_k$ such that we maximize $min(\tau_k)$ for $k\in\mathcal{K}$ while minimizing $s_{k,l} ~ \forall ~ k$. However, this problem will be non-tractable and difficult to solve. A more tractable way to achieve these objectives is to maximize the average of the local updates $\tau_k$ while controlling the staleness $s_{k,l} ~ \forall ~ k$. In this way, we can increase the number of local updates by the worst performing learner while controlling the model staleness. Based on this, the optimization variables are $\tau_k$ and $d_k$, and we can study the impact of different values of staleness by an additional constraint $s_{k,l} \leq c$. 
	
	In addition to this, the optimization needs to be done such that the global cycle is completed before time $T$ and does not violate the energy consumption limit $E_k^0$ (J) per global cycle of any learner $k$. It is observable that the relationship between the optimization variables in the global cycle time and local energy consumption constraints will be quadratic in $\tau_k$ and $d_k$. Moreover, due to $\tau_k$ and $d_k$ $\forall~k$ being non-negative integers, the resulting problem is a quadratically constrained integer linear program (QCILP) as shown: 
	
	\begin{subequations}
		\begin{align}
		&\qquad\operatornamewithlimits{max}_{\tau_k,{d}_k~\forall~k}  \quad \dfrac{1}{K}\sum_{k=1}^{K} \tau_k \tag{\ref{Eq_13_OurProb}}\\
		& \quad \nonumber\\
		\text{s.t. }\qquad & C_k^2 d_k \tau_k + C_k^1 d_k + C_k^0 \leq T, \quad \forall ~ k \in \mathcal{K} \label{orignial-time-const}\\
		& G_k^2 d_k \tau_k + G_k^1 d_k + G_k^0 \leq E_k^0, \quad \forall ~ k \in \mathcal{K} \label{original-energy-const}\\ 
		& \left\lvert \tau_k-\tau_l  \right\rvert \leq c ~ \forall k,l ~ \in \mathcal{K} \mid l \neq k \label{original-staleness-constraint}\\
		& \sum_{k = 1}^{K}d_k = d \label{orignial-batch-const}\\ 
		& \tau_k \in \mathcal{Z}_+, \quad \forall ~ k \in \mathcal{K} \label{orignial-tauk-const}\\
		& d_k \in \mathcal{Z}_+, \quad \forall ~ k \in \mathcal{K} \label{orignial-dk-const}\\
		& d_k > d_l \quad \forall ~ k \in \mathcal{K} \label{orignial-dlb-const}
		\end{align}
		\label{Eq_13_OurProb}
	\end{subequations}
	 Constraint (\ref{orignial-time-const}) and (\ref{original-energy-const}) guarantee that all learners $k$ satisfy the global cycle time constraint $T$ and the their energy consumption limit constraint $E_k^0 ~ \forall ~ k \in \mathcal{K}$, respectively.  
	 Constraint (\ref{original-staleness-constraint}) ensures that the staleness does not exceed a desired amount $c$. We will test for multiple values of the staleness and report the results later in section \ref{Sec06__SimulationResults}. The assurance that the orchestrator will learn on the complete dataset $\mathcal{D}$ is given by (\ref{orignial-batch-const}). Lastly, constraints (\ref{orignial-tauk-const}) and (\ref{orignial-dk-const}) ensure that both optimization variables are non-negative integers whereas constraint (\ref{orignial-dlb-const}) is a lower bound $d_k$ to ensure all learners are allocated a reasonable batch size. Note that the a lower bound of zero represents the case where $d_k$ can take any positive value. Solutions of (\ref{Eq_13_OurProb}) where $\tau_k$ or $d_k$ is zero for any $k$ represents a setting where learner $k$ cannot participate in the learning process and MEL may be suboptimal. 
	 
	 Thus, the program in \ref{original-energy-const} is a quadratically-constrained integer linear program (QCILP) which is well-known to be NP-hard \cite{QIP_NP_ArXiV}. This problem can be solved numerically with interior point methods using commercially available solvers. However, in the next section, we propose an analytical-numerical solution based on a combined relaxation and suggest-and-improve (SAI) approach. 
		
	\section{Proposed Solution}
	\label{Section5__Solution}
	Instead of applying the SAI technique directly on the asynchronous problem in (\ref{Eq_13_OurProb}), we first propose to solve the problem by getting candidate solutions for $d_k$ and $\tau_k$ from the synchronous problem in \cite{Moha2020b} by setting $\tau_k = \tau ~ \forall ~ k \in \mathcal{K}$. In the next step, we obtain the solution to (\ref{Eq_13_OurProb}) by applying the improve step using the candidate solutions as the initial values. The reason for doing this is because a system of $K$ learners would produce at least $\left( ^K_2 \right)$ constraints. For example, an MEL system of 100 learners will result in 4950 mutual staleness bounds. It was found that applying SAI directly to (\ref{Eq_13_OurProb}) does not work but applying the suggest step to the synchronous problem in \cite{Moha2020b} and the improve step to our problem provides solutions that converge. In the last step, the real values of the obtained $\tau_k$'s and $d_k$'s are floored to get the integer values.
		
	\subsection{Problem Relaxation}
	The problem can be simplified by replacing $\tau_k$'s with the optimization variable $\tau$ and relaxing the integer constraints in (\ref{orignial-tauk-const}) and (\ref{orignial-dk-const}) as follows:
	\begin{subequations}
		\begin{align}
		&\qquad\operatornamewithlimits{max}_{\tau,{d}_k~\forall~k}  \quad \tau \tag{\ref{Eq_13_RelaxedProblem}}\\
		& \quad \nonumber \\
		\text{s.t. }\qquad & C_k^2 d_k \tau + C_k^1 d_k + C_k^0 \leq T, \quad \forall ~ k \in \mathcal{K}  \label{relaxed-time-const}\\
		& G_k^2 d_k \tau + G_k^1 d_k + G_k^0 \leq E_k^0, \quad \forall ~ k \in \mathcal{K} \label{relaxed-energy-const}\\ 
		& \sum_{k = 1}^{K}d_k = d \label{relaxed-batch-const}\\ 
		& \tau \geq 0 \label{relaxed-tauk-const}\\
		& d_k \geq d_l, \quad \forall ~ k \in \mathcal{K} \label{relaxed-dk-const}
		\end{align}
		\label{Eq_13_RelaxedProblem}
	\end{subequations}
	Note that $\dfrac{1}{K}\sum_{k=1}^{K} \tau_k = \tau$ when $\tau_k = \tau ~ \forall ~ k \in \mathcal{K}$. The non-negative integer (\ref{orignial-dk-const}) constraint on $d_k$'s has been relaxed this and can be covered by (\ref{relaxed-dk-const}). Constraint (\ref{orignial-dk-const}) has been eliminated after relaxation because $d_l \geq 0$ in (\ref{relaxed-dk-const}). Analytically, the matrices associated with the quadratic time and energy constraints in (\ref{relaxed-time-const}) and (\ref{relaxed-energy-const}), respectively, will not be positive semi-definite due to them being sparse with two non-negative values each, which results in them having one positive and one negative eigenvalue. 
	
	\subsection{Upper Bounds using Lagrangian Relaxation}
	We can re-write the equality constraint in \ref{relaxed-batch-const} as two inequality constraints: $\sum_{k = 1}^{K}d_k - d \leq 0$ and $-\sum_{k = 1}^{K}d_k + d \leq 0$. Hence, the relaxed problem's Lagrangian function can be written as:	
	\begin{multline}\label{lagrangian}
	L\left(\mathbf{x}, \mathbf{\lambda}, \mathbf{\gamma}, \alpha, \bar{\alpha}, \omega, \mathbf{\nu}\right) = -\tau + \\ \sum_{k = 1}^{K} \lambda_k\left(C_k^2  \tau d_k + C_k^1 d_k + C_k^0 - T\right) + \\ \sum_{k = 1}^{K} \gamma_k\left(G_k^2  \tau d_k + G_k^1 d_k + G_k^0 - E_k^0\right) + \\ \alpha\left(\sum_{k = 1}^{K}d_k - d\right) - \bar{\alpha}\left(\sum_{k = 1}^{K}d_k - d\right) - \\ \omega\tau - \sum_{k = 1}^{K} \nu_kd_k    
	\end{multline}
	The Lagrange multipliers associated with the global cycle time and local energy constraints are given by $\lambda_k$ and $\gamma_k$, respectively, $\forall ~ k\in\{1,\dots,K\}$. The Lagrange multipliers related to the two total task size constraint inequalities are given by $\alpha$/$\bar{\alpha}$, and $\omega$/$\nu_k$ $k\in\{1,\dots,K\}$ are the Lagrangian multipliers associated with the non-negative constraints of both sets of optimization variables $\tau$ and $d_k$, respectively. 
	
	Let the set of optimization variables be denoted by $\mathbf{x} = \left[\tau ~ d_1 ~ d_2 ~ \ldots ~ d_k ~ \ldots ~ d_K\right]^T$ and the set of Lagrange multipliers by $\mathbf{\Gamma} = \left[\mathbf{\lambda}, \mathbf{\gamma}, \alpha, \bar{\alpha}, \omega, \mathbf{\nu}\right]^T$, where $\lambda = [\lambda_1 \ldots \lambda_k \ldots \lambda_K]^T$, $\gamma = [\gamma_1 \ldots \gamma_k \ldots \gamma_K]^T$, and $\nu = [\nu_1 \ldots \nu_k \ldots \nu_K]^T$. 
	\begin{theorem}
		The set of optimal Lagrange multipliers can be obtained by solving the dual problem in the following semi-definite program (SDP):
		\begin{subequations}
			\label{second:main}
			\begin{align}
			&\qquad\operatornamewithlimits{max}_{\mathbf{\Gamma}}  \quad \zeta \tag{\ref{second:main}}\\
			%& \quad \nonumber\\
			\text{s.t. }\qquad 
			& \nonumber \left[\begin{array}[]{cc}
			\mathbf{F}^2\left(\mathbf{\Gamma}\right) & \frac{1}{2}\mathbf{f}^1\left(\mathbf{\Gamma}\right) \\
			\frac{1}{2}\mathbf{f}^1\left(\mathbf{\Gamma}\right)  & f_0\left(\mathbf{\Gamma}\right) - \zeta 
			\end{array}\right] \succcurlyeq 0 \\
			& \nonumber \mathbf{\Gamma}  \succcurlyeq 0 
			\end{align}
			\label{Eq_10_theoremSDP}
		\end{subequations}  
	\end{theorem}
	\ignore{The functions of the set of Lagrange multipliers $\mathbf{\Gamma}$ given by $\mathbf{F}^2(\mathbf{\Gamma})$, $\mathbf{f}^1(\mathbf{\Gamma})$ and $f_0(\mathbf{\Gamma})$ are associated with the quadratic, linear, and constant terms of the objective and constraints, respectively. For details, please refer to Appendix C.}
	\begin{proof}
		The optimization variables be denoted by $\mathbf{x}$ where $\mathbf{x} = \left[\tau ~ d_1 ~ d_2 ~ \ldots ~ d_k ~ \ldots ~ d_K\right]^T$. Then, the relaxed program in (\ref{Eq_13_RelaxedProblem}) for the synchronous case can be re-written in the form of a QCQP as shown below:
		\begin{subequations}
			\begin{align}
			&\qquad\operatornamewithlimits{min}_{\mathbf{x}}  \quad \mathbf{x}^T \mathbf{F} \mathbf{x} +\mathbf{f}^T \mathbf{x} + f_0\\
			%& \quad \nonumber \\
			\text{s.t. }\qquad & \mathbf{x}^T \mathbf{P}_k \mathbf{x} +\mathbf{p}_k^T \mathbf{x} + p_k^0 \leq 0, \quad \forall k \in \mathcal{K}    \label{matrix-time-const}\\
			& \mathbf{x}^T \mathbf{Q}_k \mathbf{x} +\mathbf{q}_k^T \mathbf{x} + q_k^0 \leq 0, \quad \forall k \in \mathcal{K}  \label{matrix-energy-const}\\ 
			& \mathbf{x}^T \mathbf{A} \mathbf{x} +\mathbf{a}^T \mathbf{x} + a_0 \leq 0 \label{matrix-batch-const}\\
			& \mathbf{x}^T \bar{\mathbf{A}} \mathbf{x} +\bar{\mathbf{a}}^T \mathbf{x} + \bar{a}_0 \leq 0 \label{matrix-batch-const-neg}\\  
			& \mathbf{x}^T \mathbf{U} \mathbf{x} +\mathbf{U}^T \mathbf{x} + u_0 \leq 0 \label{matrix-tau-const}\\
			& \mathbf{x}^T \mathbf{V}_k \mathbf{x} +\mathbf{v}_k^T \mathbf{x} + v_k^0 \leq 0, \quad \forall k \in \mathcal{K} \label{matrix-d-const}
			\end{align}
			\label{Eq_13_MatrixProblem}
		\end{subequations}
		
		Constraints (\ref{matrix-time-const}) and (\ref{matrix-energy-const}) represent the time and energy constraints, respectively, and constraints (\ref{matrix-batch-const}) and (\ref{matrix-batch-const-neg}) represent the total batch size constraint as two inequalities. The non-negative constraints on $\tau$ and $d_k$ are given in (\ref{matrix-tau-const}) and (\ref{matrix-d-const}), respectively. The constants associated with the time and energy constraints can be defined as $p_k^0 = C_k^0-T$ and $q_k^0 = G_k^0-E_k^0$, respectively, $\forall ~ k$. The variables $a_0 = -d$, $\bar{a}_o = d$ and $v_k^0 = d_l, \forall ~ k$ whereas $u_0 = 0$ and $f_0  =0$. 
		
		The coefficients associated with the linear terms in the objective and constraints ($\mathbf{f}$ and $~\mathbf{p}_k\text{,} ~\mathbf{q}_k\text{,} ~\mathbf{a}\text{,} ~\bar{\mathbf{a}}\text{,} ~\mathbf{u}, \text{ and } \mathbf{v}_k$, respectively) are given in (\ref{Eq_14_vector_defs})as column vectors. 
		\begin{subequations}
			\label{third:main}
			\begin{align}
			&\quad \nonumber \mathbf{f} = \left[-1 ~ 0 ~ 0 ~ \ldots ~ C_k^1 ~ \ldots ~ 0\right]^T \tag{\ref{third:main}}\\
			&\quad \nonumber \mathbf{p}_k = \left[0 ~ 0 ~ 0 ~ \ldots ~ C_k^1 ~ \ldots ~ 0\right]^T, \forall ~ k\\
			&\quad \nonumber \mathbf{q}_k = \left[0 ~ 0 ~ 0 ~ \ldots ~ G_k^1 ~ \ldots ~ 0\right]^T, \forall ~ k\\
			&\quad \nonumber \mathbf{a} = \left[0 ~ 1 ~ 1 ~ \ldots ~ 1 ~ \ldots ~ 1\right]^T\\
			&\quad \nonumber \bar{\mathbf{a}} = \left[0 ~ -1 ~ -1 ~ \ldots ~ -1 ~ \ldots ~ -1\right]^T\\
			&\quad \nonumber \mathbf{u} = \left[-1 ~ 0 ~ 0 ~ \ldots ~ 0 ~ \ldots ~ 0\right]^T\\
			&\quad \nonumber \mathbf{v}_k = \left[0 ~ 0 ~ 0 ~ \ldots ~ -1 ~ \ldots ~ 0\right]^T, \forall ~ k
			\end{align}
			\label{Eq_14_vector_defs}
		\end{subequations} 	
		The quadratic matrices associated with the time and energy constraints, $\mathbf{P}_k$ and $\mathbf{Q}_k$, respectively, are given in (\ref{eq_P_matrix}) and (\ref{eq_Q_matrix}).
		\begin{equation}
		\mathbf{P}_k(i,j) = \begin{cases} 0.5C_k^2, \mbox{ if}  & \begin{matrix}
		i = 1 ~ \&  ~ j = k+1 \\ i = k+1 ~ \& ~ j = 1
		\end{matrix}\\ 0, & \mbox{otherwise} \end{cases}
		\label{eq_P_matrix}
		\end{equation}	
		\begin{equation}
		\mathbf{Q}_k(i,j) = \begin{cases} 0.5G_k^2, \mbox{ if}  & \begin{matrix}
		i = 1 ~ \&  ~ j = k+1 \\ i = k+1 ~ \& ~ j = 1
		\end{matrix}\\ 0, & \mbox{otherwise} \end{cases}
		\label{eq_Q_matrix}
		\end{equation}	
		The remaining quadratic matrices $\mathbf{F}$, $\mathbf{A}$, $\mathbf{\bar{A}}$, $\mathbf{U}$ and $\mathbf{V}_k$ are all $\mathbf{0}_{(K+1) \times (K+1)}$. 
		
		The functions $\mathbf{F}^2(\mathbf{\Gamma})$, $\mathbf{f}^1(\mathbf{\Gamma})$ and $f_0(\mathbf{\Gamma})$ can now be defined as \cite{QCQP_Boyd}:	
		\begin{equation}
		\mathbf{F}^2(\mathbf{\Gamma}) = \sum_{k = 1}^{K} \lambda_k \mathbf{P}_k + \gamma_k \mathbf{Q}_k
		\label{eq_func_quad}
		\end{equation}
		\begin{equation}
		\mathbf{f}^1(\mathbf{\Gamma}) = \sum_{k = 1}^{K}\left( \lambda_k \mathbf{p}_k + \gamma_k \mathbf{q}_k + \nu_k \mathbf{v}_k\right) + \alpha \mathbf{a} + \bar{\alpha} \bar{\mathbf{a}} + \omega \mathbf{u}
		\label{eq_func_linr}
		\end{equation}
		\begin{equation}
		f_0(\mathbf{\Gamma}) = \sum_{k = 1}^{K}\left( \lambda_k p_k^0 + \gamma_k q_k^0 + \nu_k v_k^0 \right) + \alpha a_0 + \bar{\alpha} \bar{a}_0 
		\label{eq_func_cons}
		\end{equation}
	\end{proof}
	
	A candidate solution to the SDP in (\ref{Eq_10_theoremSDP}) is given by:
	\begin{equation}
		\hat{\mathbf{x}} = -\dfrac{1}{4}\mathbf{F}^2\left(\mathbf{\Gamma}\right)^{-1}\mathbf{f}^1\left(\mathbf{\Gamma}\right)
		\label{Eq_11_Candidate}
	\end{equation}
	The candidate solution will be the optimal solution in case of a convex QCQP. In the case of a non-convex QCQP, there is an expected duality gap and an improve step is needed. Hence, for the synchronous approach, we apply this step to the problem in (\ref{Eq_13_RelaxedProblem}). For the proposed asynchronous approach, the optimal solution can be obtained by applying the local optimizer coordinate descent (CD) to the problem in (\ref{Eq_13_OurProb}) with the relaxed constraints in (\ref{relaxed-tauk-const}) and (\ref{relaxed-dk-const}) with $\tau_k$ replacing $\tau$ in (\ref{relaxed-tauk-const}) and the additional constraint in (\ref{original-staleness-constraint}). 
	The complete steps followed by the orchestrator over multiple global cycles are summarized in Algorithm \ref{alg1}.
	\begin{algorithm}[t]
		\caption{Process at the Orchestrator}
		\label{alg1}
		\begin{algorithmic}[1]
			\renewcommand{\algorithmicrequire}{\textbf{Input:}}
			\renewcommand{\algorithmicensure}{\textbf{Output:}}
			\REQUIRE $T$, $d$, $d_l$, $K$
			\ENSURE  $\mathbf{w}$
			\\Initialize $\mathbf{w}$ and set the flag \textit{STOP} $\leftarrow$ \textbf{FALSE}
			\WHILE {\NOT \textit{STOP}}
			\STATE \textit{In Parallel}: Send $\mathbf{w}$ to each learner $k \in \mathcal{K}$ 
			\STATE \textit{In Parallel}: Receive $P_{kO}$, $h_{kO}$, $f_k$, and $e_k^0$ from $k \in \mathcal{K}$
			\STATE Solve (\ref{Eq_11_Candidate}) to obtain $\hat{\tau}$, $\hat{d_k}$
			\STATE Transform (\ref{Eq_13_OurProb}) by setting $\tau_k > 0 ~ \forall ~ k \in \mathcal{K} $ in (\ref{orignial-tauk-const}) and removing (\ref{orignial-dk-const})
			\STATE Get $\tau_k$ and $d_k$ by applying CD to (\ref{Eq_13_OurProb}) using $\hat{\tau}$, $\hat{d_k}$
			\STATE \textit{In Parallel}: Send $\lfloor \tau_k \rfloor$, $\lfloor d_k \rfloor$ to each learner $k \in \mathcal{K}$\footnotemark
			\STATE \textit{In Parallel}: After $\tau_k$ local updates, receive $w_k ~ \forall ~ k \in \mathcal{K}$
			\STATE Obtain $\mathbf{w}$ using (\ref{Eq_Sec03_GlobalAgg})
			\IF {\textit{STOPPING CRITERIA REACHED}}
			\STATE Set \textit{STOP} $\leftarrow$ \textbf{TRUE}
			\ENDIF
			\ENDWHILE
			\RETURN $\mathbf{w}$ 
		\end{algorithmic} 
	\end{algorithm} 
	\begin{table}[t!]
		\centering
		\small
		\caption{List of simulation parameters}
		\begin{tabular}{|l|l|}
			\hline
			Parameter                                 		& Value                         						\\ \hline
			Cell Attenuation Model							& $128+37.1\log(R)$ dB \cite{Moha1812:Multi}         	\\ \hline
			Channel Bandwidth $(W)$							& 5 MHz                               					\\ \hline
			Device proximity (R)							& 500m													\\ \hline
			Transmission Power $(P_k)$                   	& 23 dBm                              					\\ \hline
			Noise Power Density $(N_0)$                  	& -174 dBm/Hz                         					\\ \hline
			Computation Capability $(f_k)$             		& $\sim \{6.0,~2.4,~1.4,~0.7\}$ GHz             		\\ \hline
			MNIST Dataset size (d)                	 		& 60,000 images                        					\\ \hline
			MNIST Dataset Features $(\mathcal{F})$  		& 784 ($~28 \times 28~$) pixels     					\\ \hline
		\end{tabular}
		\label{Table_1_OfParameters}
	\end{table}
	\begin{figure}[t!]
		\centering
		\includegraphics[scale=0.6]{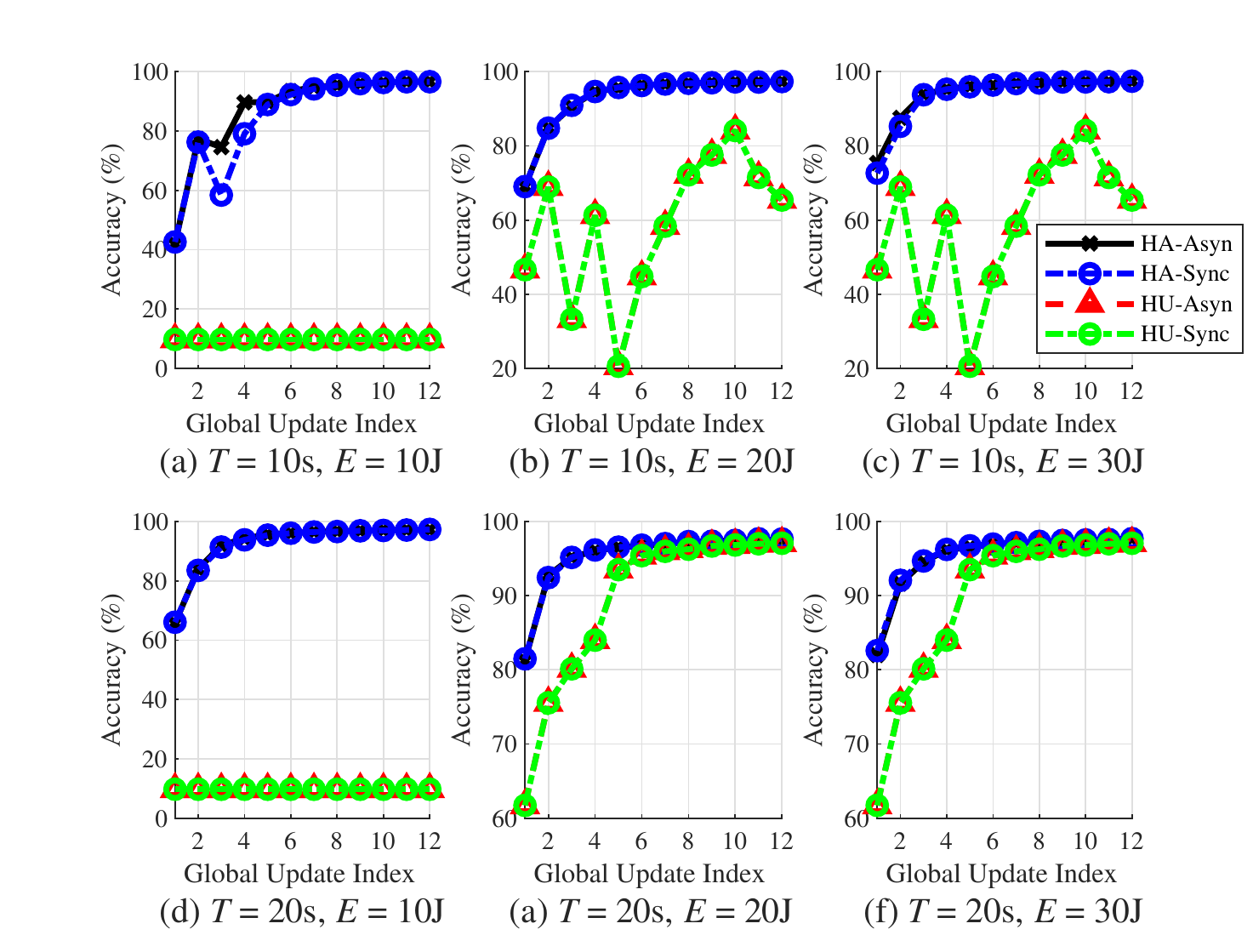}
		\caption{Final validation accuracy achieved after a total of 12 global epochs with various training times for an average device energy consumption of 10J.}
		\label{figure06label}
	\end{figure}

	\section{Simulation Results}
	\label{Sec06__SimulationResults}

	\subsection{Simulation Environment}
	\begin{figure*}[t!]
		\centering
		\includegraphics[width=\linewidth,height=\textheight,keepaspectratio]{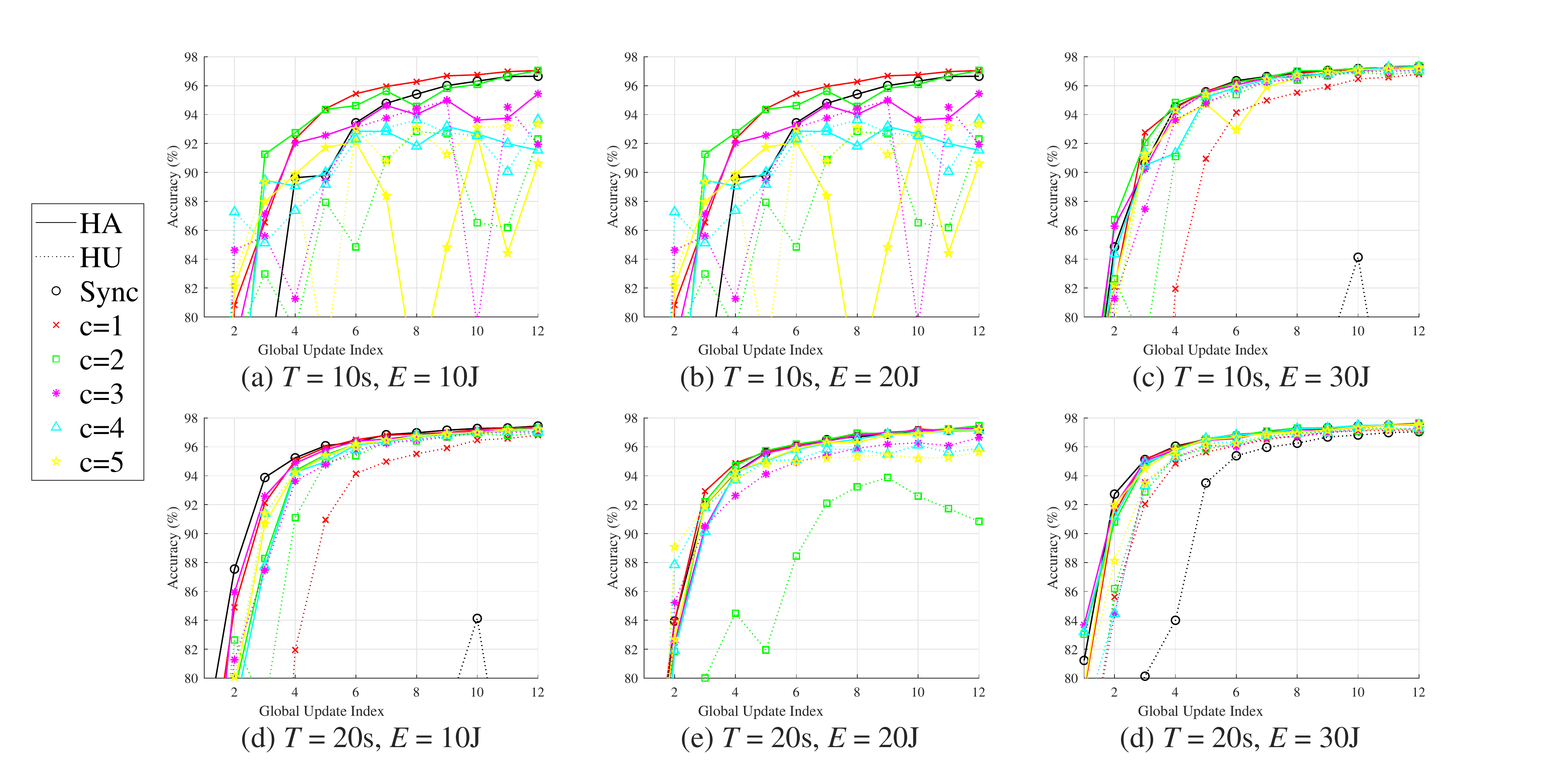}
		\caption{Accuracy progression for the cases when $T$ = 10s and 20s and for average energy values of $E=\{10,20,30\}J$.}
		\label{figure07label}
	\end{figure*}
	It is assumed that the learners comprise a combination of the following: laptops with multi-core processors, smart phones simple processors, advanced micro-controllers such as the Raspberry Pi, and very simple micro-controllers such as the Arduino. The learners are co-located within 500 meters in a cellular type environment and may be mobile. The channel parameters and other specs are listed in table \ref{Table_1_OfParameters}. To test our proposed MEL paradigm, a fully-connected deep neural network consisting of 300, 124, and 60 neurons is used to train on the MNIST dataset \cite{MNIST_IEEE}. For detailed descriptions on how to obtain the model size and computational complexity, the readers are referred to \cite{Mohammad2019a}.
	\footnotetext{In PL, the orchestrator sends $d_k$ samples after randomly shuffling its dataset whereas it sends $d_k$ in FL and the learner chooses $\min(d_k,d_k^{max})$ data points where $d_k^{max}$ is learner $k$'s dataset size.}
	
	The two major additional contributions of this paper are the study of different levels of caps on the energy consumed per global cycle per learner $k$ and the impact of different values of staleness capped by $c$. 
	We plot the validation accuracy related metrics for values of staleness of up-to $c=5$, because it was found that having a higher $c>5$ does offer any improvements.   
	
	As for the constraint on the local energy consumption $E_k^0 ~ \forall ~ k \in K$, it is possible that devices will have wildly varying consumption limits. However, to quantify the impact of these limits, we define an average energy consumption per global cycle across all learners $E$ (J) and $E_k^0$ (J) in any global cycle varies by $\sigma_k^0$ (J) such that $E_k^0 = E \pm \sigma_k U ~ \forall ~ k\in\mathcal{K}$ where $U \sim \mathcal{U}(0,1)$ (J). To put these numbers into perspective, modern batteries are rated in terms of voltage (V) and milliampere-hours (mAH). An average consumption of 20J per global cycle for 10 cycles would imply a total consumption of 200J. For a battery rated at 5V, this represents a consumption of 11.1 mAH which for a 2000 mAH battery, represents 0.36\% of the maximum load.  
	
	In the following sub-sections, we present results for an MEL system comprising $K=$ 20 learners tested for global cycle times of $T=5$s to $T=40$s in steps of 5 seconds and for average energy values of $E = \{10,~20,~30\}$J with $\sigma_k^0 = 2.5$ (J). The results are discussed for our proposed HA asynchronous (HA-Asyn) scheme with PL and compared to the HA synchronous (HA-Sync) scheme in \cite{Mohammad2019a} and the performance that would have been achieved if the HU random equal task allocation approach was used (HU-Sync/Asyn) such as the one described in \cite{Wang2019}.   
		
	\subsubsection{Convergence Proof}
	\begin{figure*}[t!]
		\centering
		\includegraphics[width=\textwidth,keepaspectratio]{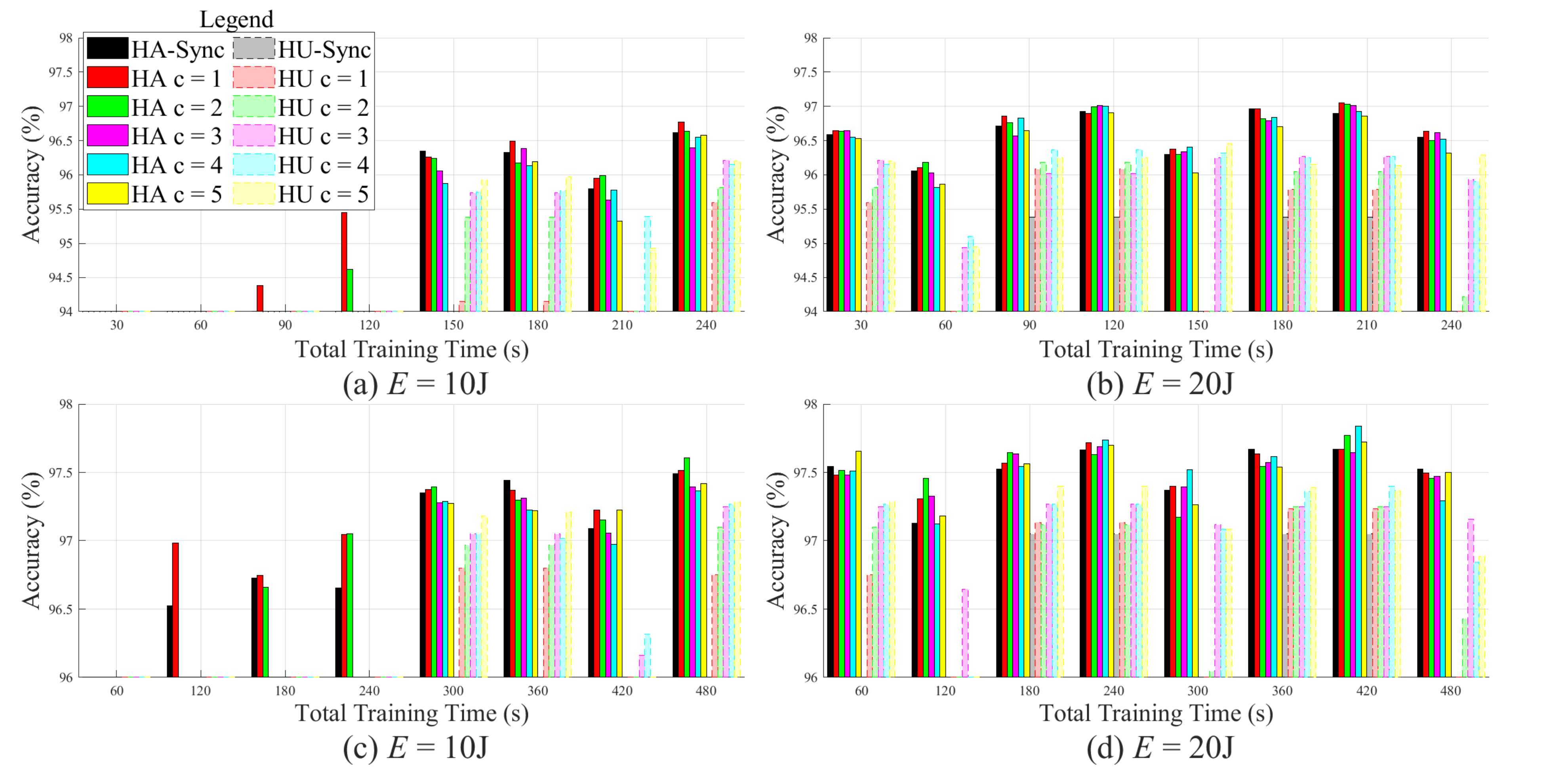}
		\caption{Final validation accuracy values after 6 and 12 global updates, respectively, for various training times for average energy consumption limits $E=10$J and $E=20$J, respectively.}
		\label{figure14label}
	\end{figure*}
	\begin{figure*}[t!]
		\centering
		\includegraphics[width=\textwidth,,height=\textheight,keepaspectratio]{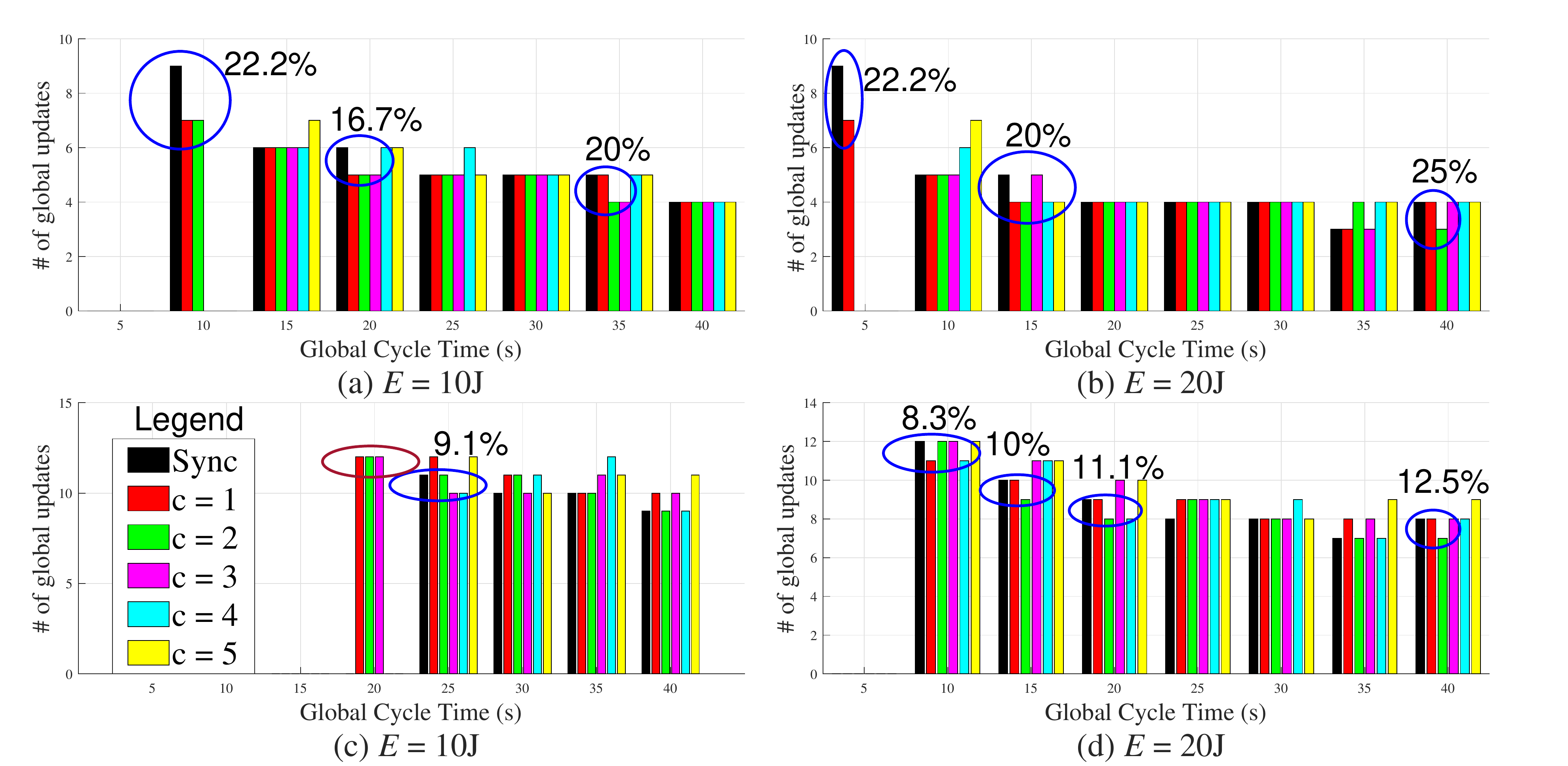}
		\caption{Final validation accuracy values after 6 and 12 global updates, respectively, for various training times for average energy consumption limits $E=10$J and $E=20$J, respectively.}
		\label{figure15label}
	\end{figure*}

	Figure \ref{figure06label} shows the plots for learning accuracy for the synchronous case only where HA-Sync represents solution obtained using the approach in \cite{Moha2020b} and HA-Asyn represents solutions to problem (\ref{Eq_13_RelaxedProblem}) with $c$ set to zero. The HU-Sync/Asyn plots represent the HU solution for both approaches. As observable, we have confirmed that our solution in this paper converges to the synchronous case. Moreover, for extremely low values of energy and time (10J and 10s, respectively), the proposed approach works better than the approach to solve the synchronous problem. Last but not least, for the synchronous case ($c=0$), the HA scheme always works better than the HU scheme where the HU scheme fails to converge on several occasions. 	
	
	Figure \ref{figure06label} plots the learning accuracy progression for different values of staleness (including the synchronous case with $c=0$) with varying values of $T$ and $E$ for both, the HA and HU schemes. The purpose of these figures is to demonstrate the importance of utilizing HA schemes, especially when the resources are limited. Observe that the HA schemes generally converge faster and reach a higher level of final validation accuracy. The plots also demonstrate the usefulness of having a staleness aware asynchronous scheme. 
	
	For example, as we can see from Fig. \ref{figure07label}a that when $T=10$s and $E=10$J, the HA-Asyn with $c=1$ and $c=2$ requires 5 global updates to reach a 94\% accuracy whereas the HA-Sync requires 7 updates, a reduction in time of 29\%. Similarly, 6 updates are required to achieve a 95\% accuracy with $c=1$ whereas the HA-Sync needs 8 updates representing a reduction of 25\%. Because this is not clear from Fig. \ref{figure07label} in general, the next subsection demonstrates these gains more clearly using bar charts.
	
	\subsubsection{Validation Accuracy}
	\begin{figure*}[t!]
		\centering
		\includegraphics[width=\textwidth,,height=\textheight,keepaspectratio]{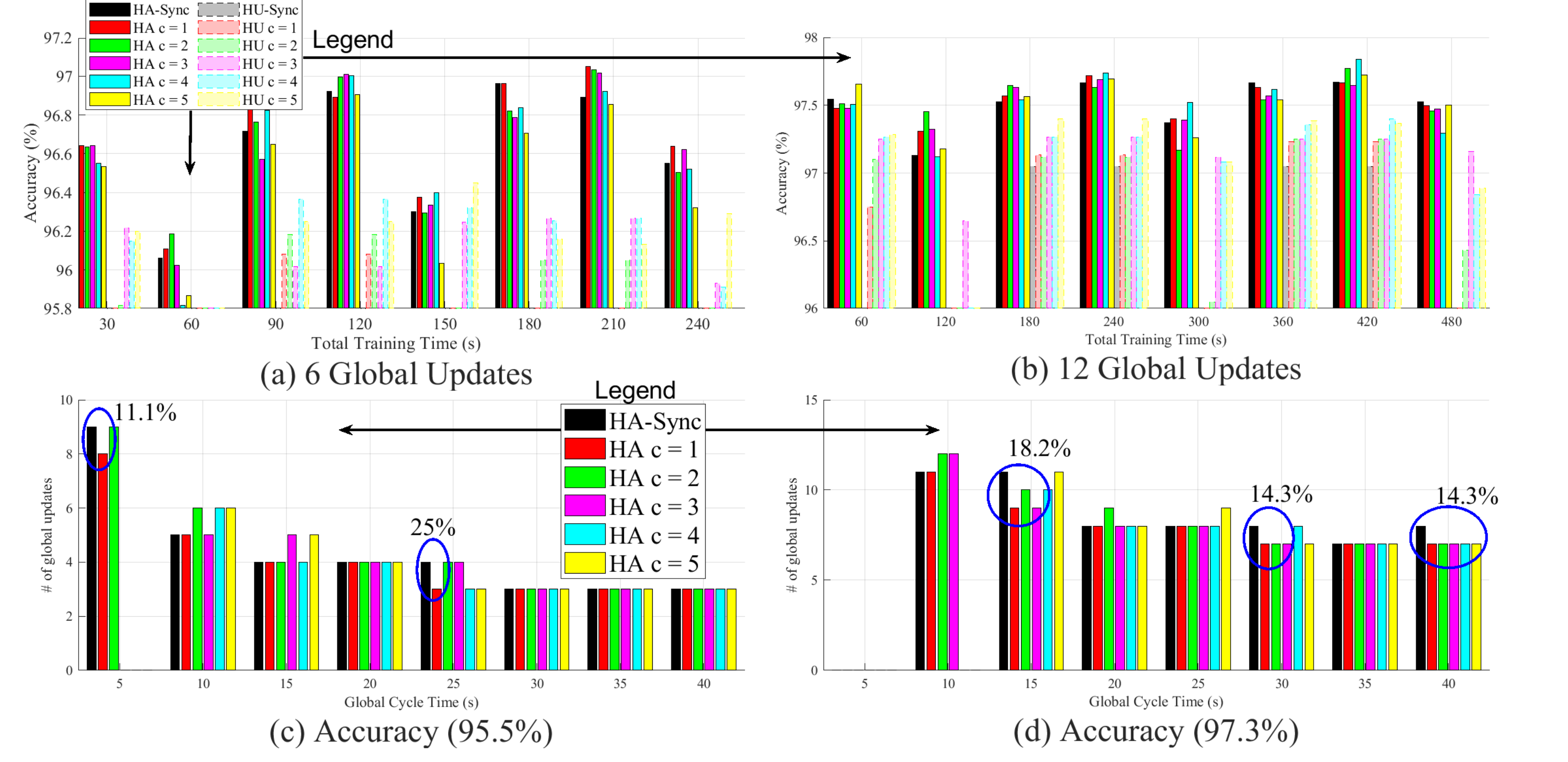}
		\caption{Results for an MEL system with average energy $E=30$J. Final validation after (a) 6 global updates and (b) 12 global updates, and number of global updates to reach an accuracy threshold of (a) 95.5\% and (b) 97.3\%.}
		\label{figure16label}
	\end{figure*}
	
	In this part, we focus on some specific metrics such as final validation accuracy after a set number of global updates or the number of updates required to reach an accuracy threshold. These results are presented for all schemes for the case where the devices have a low average energy consumption of about 10J and a higher consumption of 20J per global cycle for each learner for all global cycle times.  This study is important because less number of updates for a given global cycle time constraint results in lower total training time for a given time constraint.
	
	Figs. \ref{figure14label}a and \ref{figure14label}b present the final validation accuracy achieved after a total of 6 global updates with different total training times for the case when the average energy consumption limit per cycle per learner $E=10$J and $E=20$J, respectively. The final validation accuracies for those two settings after 12 global cycles is given in  Fig. \ref{figure14label}c and Fig. \ref{figure14label}d, respectively. The case when $c = 0$ implies the HA/HU-Sync scheme and cases when $c>0$ represent the HA/HU-Asyn schemes. 
	
	Overall, it can be seen that the HA schemes provide a better performance compared to the best performance of the HU schemes. In most regions, the best performance is provided by the HA-Asyn schemes. For example, in the low resource region with $T = 15$ and $20$s and $E = 10$J, after 6 global updates the HA-Asyn with $c=1$ provides a best accuracy of 94.4\% and 95.4\%, respectively, whereas the HA-Sync fails to reach 94\%. The HA-Asyn with $c=2$ crosses 94.5\% for $T=20$s. For a higher energy of $20$J, the HA-Asyn scheme with $c=$ 1-3 is able to provide an accuracy of 97\% when $T=20$s and for $c$ values of 2-4 for $T=35$s. After 12 global cycles, for $E=10$J and $T=10$s and $20$s, the HA-Async schemes with $c=1$ and $c=\{1,2\}$, respectively, are able to provide a significantly higher final accuracy with a difference of more than 0.4\%. Similarly when $E=20$J, for global cycle times of T=$\{5,~,10,~15,~20,~25\}$s, the HA-Asyn scheme provides the best validation accuracy with corresponding staleness $c=\{5,2,2,4,4\}$. with the difference ranging from 0.1-0.3\%.
	
	Fig. \ref{figure15label} displays these differences clearly by showing the number of updates required to reach a 95.5\% accuracy for various global cycle times $T$ for $E=10$J (Fig. \ref{figure15label}a) and $E=20$J (Fig. \ref{figure15label}b). For the same settings, the number of updates required to achieve an accuracy of 97.3\% are plotted in Figs. \ref{figure15label}c and Fig \ref{figure15label}d, respectively. For example, when $E=10$J and $T=10$s, a final accuracy of 95\% can be achieved in 7 global updates with HA-Asyn as opposed to 9 updates with the HA-Sync, representing a reduction of 22.2\%. This means that the MEL system needs to train for 20s less to reach the same accuracy. Similarly, after 12 global update cycles, an MEL system with $E=10$J can reach an accuracy of 97.3\% with HA-Asyn ($c=1,2,3$) whereas the HA-Sync cannot reach that level of accuracy. Moreover, for $T=25$s, the HA-Asyn ($c=3,4$) can 97.3\% in 10 updates as compared to 11, representing a reduction of 9.1\% or 25s.
	
	\red{change }Fig. \ref{figure16label} presents the results for the high energy region of $E=30$J. Figs. \ref{figure16label}a and \ref{figure16label}b show the final validation accuracy achieved by the HA  
	
	\subsection{Discussions}
	Despite these gains, it can be seen that in some situations, the HA-Sync approach provides the best results. For example, when $T=25$s and $E=20$J, the least time to reach an accuracy of 97.2\% is by the HA-Sync scheme. It also reaches an accuracy of 95.5\% and 97\% with the same number of updates for values $T$ in the range $25-35$s for $E=10$J and $20$J. Moreover, the best final validation accuracy after 12 global cycles is provided by HA-Sync for $E=10$J when $T=30$s and also for $E=20$J when $T=30$s and $T=40$s. 
	
	Although it is difficult to see a concrete trend, we may conclude that the HA-Asyn works best when the resources are at their lowest and the synchronous approach may fail, It can also provide gains when one resource is low and the other high, especially, when energy is abundant and time is low. This works for the scenarios where FL has been proposed for devices that are charging and not on battery power. In the medium range of both resources, time and energy, the HA-Sync works best but this not typical of the edge environment. When both resources are abundant, also not typical but may happen for example, if the main objective is privacy and not minimal delay, where learning takes place on devices that are charging. 
	
	To conclude, we suggest the following two-step process to select the best scheme out of the HA-Sync and HA-Asyn:
	\begin{enumerate}
		\item If either resource, time or energy is low, choose the HA-Asyn approach. If both resources are in the medium range, then go for the HA-Sync. If both resources are high, choose the HA-Asyn approach.
		\item If the HA-Sync is chosen, simply do cross-validation on the ML model and associated hyper-parameters such as model size, learning rate, regularization, etc. On the other hand, when HA-Asyn is used, the parameter $c$ should be added to the cross-validation. In the very low resource region, checking with $c=1$ $c=2$ may suffice whereas in case when one or more of the resources are abundant, a set from the range $c\geq3$ may be used. 
	\end{enumerate}
	
	\section{Conclusion}
	\label{Sec07__Conclusion}
	This paper extends the research efforts towards establishing the novel MEL paradigm by proposing a HA asynchronous (HA-Asyn) approach.
	It was shown that for an MEL system with learners performing an asynchronous number of updates, maximizing the average number of updates while controlling the maximum mutual staleness can improve accuracy. 
	A two-step solution based on the SAI method was designed and through extensive simulations using the well-known MNIST dataset, it was shown that the proposed HA-Asyn scheme provides better validation accuracy and a faster progression than the HU scheme and for many scenarios, provides a better performance than the HA-Sync scheme with gains of up-to 25\%. Finally, strategies were proposed to select between the HA-Sync and HA-Asyn schemes with recommendations on cap the staleness.
	
	\ifCLASSOPTIONcaptionsoff
	\newpage
	\fi
	
	%\balance 
	\bibliographystyle{IEEEtran}
	\bibliography{main}
	
	\begin{IEEEbiography}[{\includegraphics[width=1in,height=1.25in,clip,keepaspectratio]{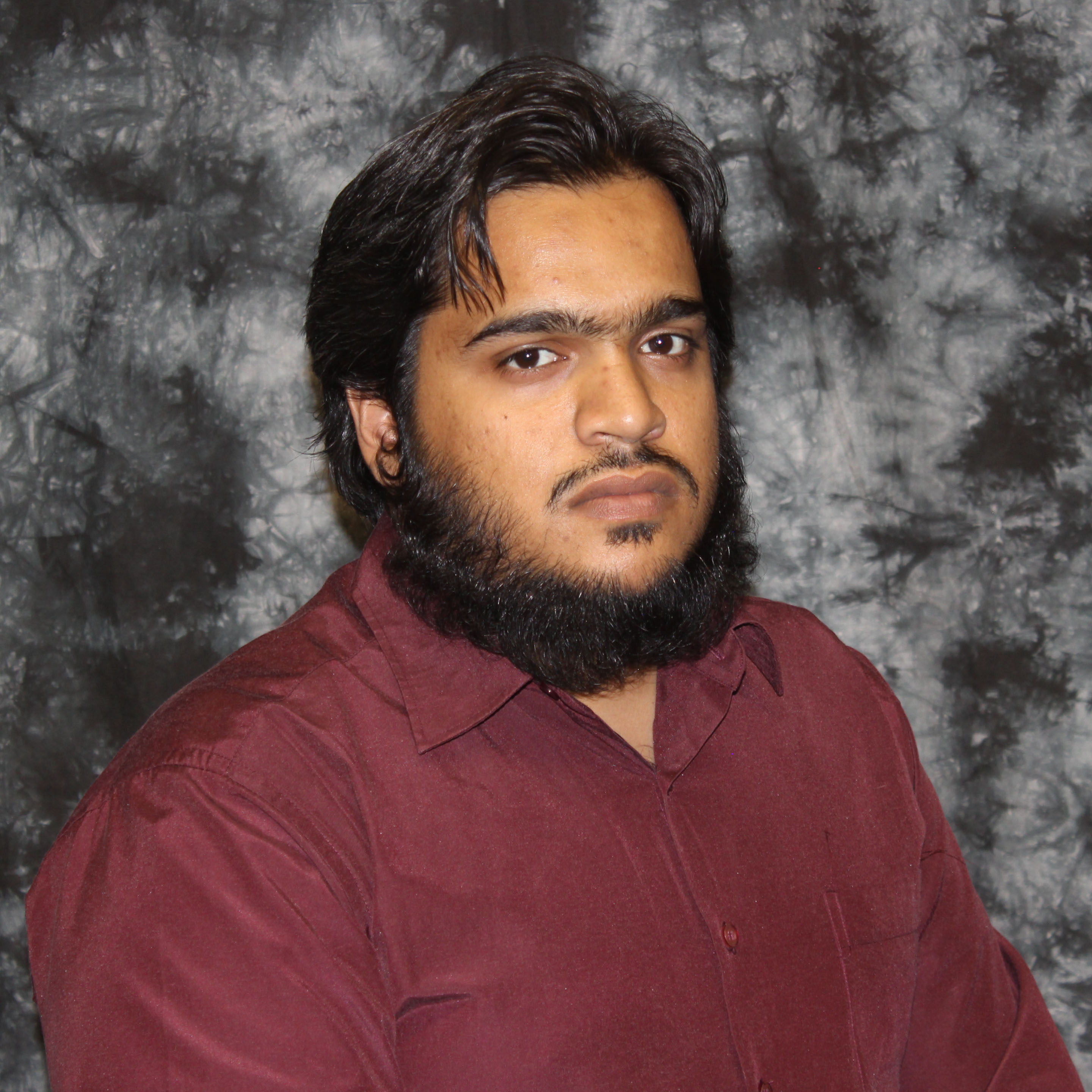}}]{Umair Mohammad} (S '12) Umair Mohammad received his Bachelor's degree in Electrical Engineering and Master's degree in Telecommunication Engineering from the King Fahd University of Petroleum and Minerals (KFUPM) in 2013 and 2016, respectively. Currently, Umair is a PhD candidate in the department of Electrical and Computer Engineering at the University of Idaho (UI) and a research assistant (RA) for the National Institute for Advanced Transportation Technology (NIATT). Umair’s areas of interest include wireless communication, edge computing, machine learning (ML) and distributed ML for wireless edge networks.
	\end{IEEEbiography}
	\begin{IEEEbiography}[{\includegraphics[width=1in,height=1.25in,clip,keepaspectratio]
			{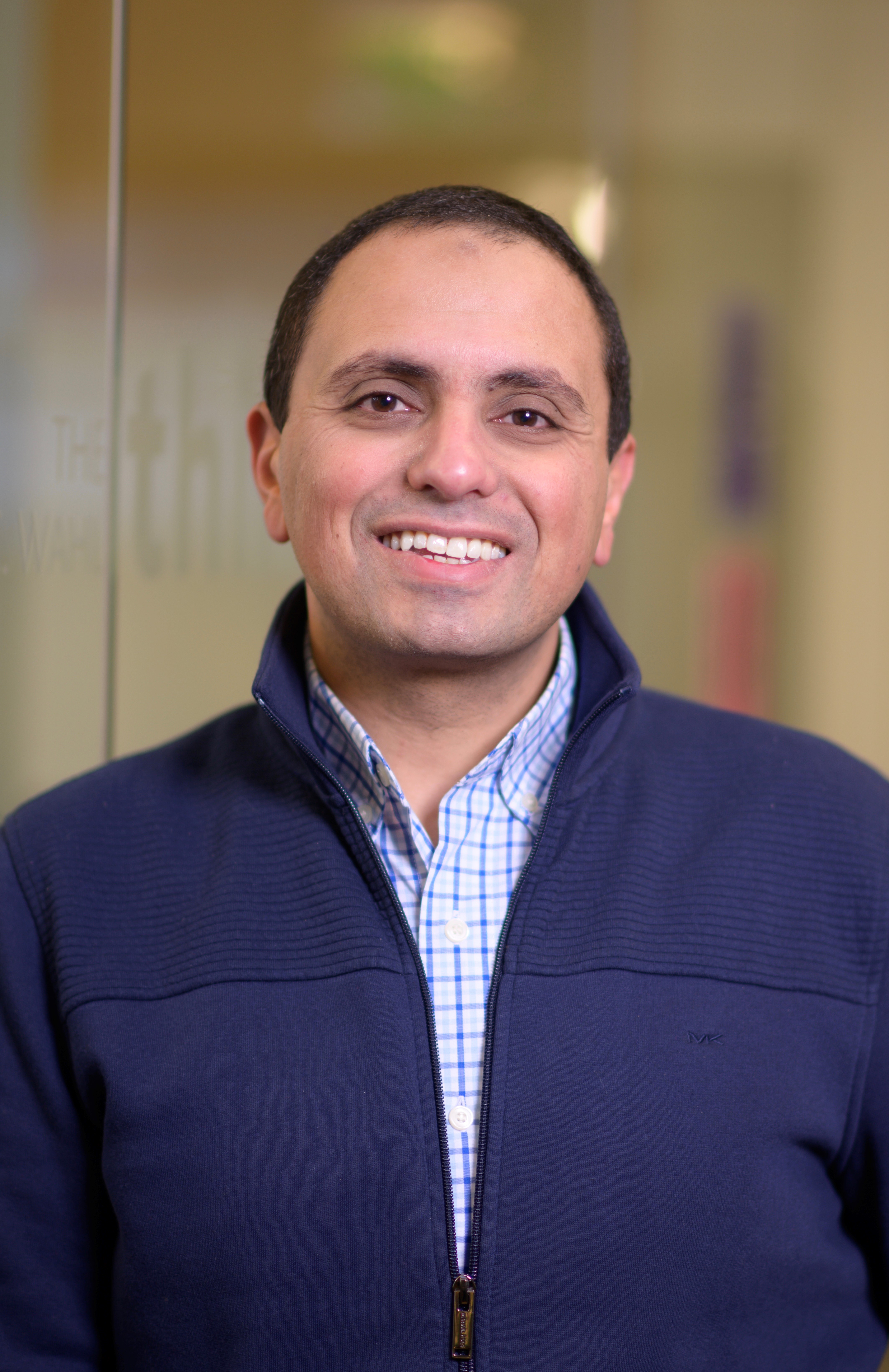}}]{Sameh Sorour} (S '98, M '11, SM'16) is an Assistant Professor at University of Idaho, USA. He received his B.Sc. and M.Sc. degrees from Alexandria University in 2002 and 2006, respectively, and his PhD from University of Toronto in 2011. His PhD thesis was nominated for the Governor General’s Gold Medal Award. After his graduation, he held a MITACS industrial postdoctoral fellowship with Siradel Canada and University of Toronto. Prior to moving to University of Idaho in 2016, he held another postdoctoral fellowship at King Abdullah University of Science and Technology (KAUST), and an assistant professor position at King Fahd University of Petroleum and Minerals (KFUPM). During his PhD and postdoctoral fellowships, he led several research projects with industrial partners and government agencies, such as LG Korea, the European Space Agency, the Canadian National Institute for the Blind (CNIB), and Siradel France. Dr. Sorour is currently a senior IEEE member and an Editor for IEEE Communications Letters. His research and educational interests lie in the broad areas of advanced computing, learning, and networking technologies for cyber-physical and autonomous systems. Topics of particular interest include cloud/edge/IoT networking, computing, learning, and their applications in multimodal/coordinated autonomous driving, autonomous/electric mobility on demand systems, layered/virtualized management of future transportation networks, cyber-physical systems, and smart energy and healthcare systems.
	\end{IEEEbiography}
	\begin{IEEEbiography}[{\includegraphics[width=1in,height=1.25in,clip,keepaspectratio]
			{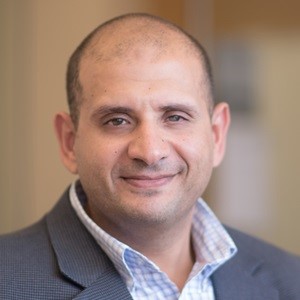}}]{Mohamed Hefeida} is a clinical assistant professor with the ECE department at the University of Idaho. He received his doctorate in electrical and computer engineering from the University of Illinois at Chicago in 2013, where he also worked as a visiting assistant professor for one year. Before joining the ECE department at the University of Idaho, Dr. Hefeida was an assistant professor of electrical engineering at the American University of the Middle East in Kuwait. His research interests span a wide spectrum of networking and communication techniques, with particular interest in Wireless Sensor Networks, Internet of Things (IoT), Cross-Layer Design, and Data Management.
	\end{IEEEbiography}
	
	%\newpage

	
	%\vfill
	


\end{document}